\documentclass[english,a4paper,12pt]{article}

\PassOptionsToPackage{pdfpagelabels=false}{hyperref}

\usepackage{amsmath, amsxtra, amsfonts, amssymb, amstext}
\usepackage{amsthm}
\usepackage{booktabs}
\usepackage{nicefrac}
\usepackage{xspace}
\usepackage[noadjust]{cite}
\usepackage{url}
\usepackage{graphics}
\usepackage[usenames,dvipsnames]{xcolor}
\usepackage[colorlinks]{hyperref}
\definecolor{linkblue}{rgb}{0.1,0.1,0.8}
\hypersetup{colorlinks=true,linkcolor=linkblue,filecolor=linkblue,urlcolor=linkblue,citecolor=linkblue}
\usepackage[algo2e,ruled,vlined,linesnumbered]{algorithm2e}
\usepackage{wrapfig}

\allowdisplaybreaks[4]
\newtheorem{theorem}{Theorem}
\newtheorem{lemma}[theorem]{Lemma}
\newtheorem{corollary}[theorem]{Corollary}
\newtheorem{definition}[theorem]{Definition}

\newcommand{\oea}{\mbox{$(1 + 1)$~EA}\xspace}
\newcommand{\oeamu}{$(1 + 1)$~EA$_{\mu,p}$\xspace}

\newcommand{\oplea}{\mbox{$(1+\lambda)$~EA}\xspace}
\newcommand{\mpoea}{\mbox{$(\mu+1)$~EA}\xspace}
\newcommand{\mplea}{\mbox{$(\mu+\lambda)$~EA}\xspace}

\newcommand{\OM}{\textsc{OneMax}\xspace}
\newcommand{\onemax}{\OM}
\newcommand{\LO}{\textsc{Leading\-Ones}\xspace}
\newcommand{\leadingones}{\LO}
\newcommand{\needle}{\textsc{Needle}\xspace}
\DeclareMathOperator{\jump}{\textsc{jump}}
\DeclareMathOperator{\rand}{rand}

\newcommand{\R}{\ensuremath{\mathbb{R}}}

\newcommand{\N}{\ensuremath{\mathbb{N}}} 

\newcommand{\calA}{\ensuremath{\mathcal{A}}} 
\newcommand{\calS}{\ensuremath{\mathcal{S}}} 
\newcommand{\calT}{\ensuremath{\mathcal{T}}} 
	
\DeclareMathOperator{\Bin}{Bin}
\DeclareMathOperator{\Geom}{Geom}

\DeclareMathOperator{\mutate}{mutate}
\newcommand{\pmin}{p_{\mathrm{min}}}

\newcommand{\eps}{\varepsilon}

\newcommand{\assign}{\leftarrow}

\begin{document}
{\sloppy

\title{Better Runtime Guarantees \\Via Stochastic Domination\footnote{Extended version of a paper that appeared at \emph{EvoCOP 2018}~\cite{Doerr18evocop}. This version contains as new material a section on known precise runtime distributions, a Chernoff bound for sums of independent coupon collector runtimes, several new tail bounds for classic runtime results, and a section on counter-examples.}}

\author{Benjamin Doerr\\ \'Ecole Polytechnique\\ CNRS\\ Laboratoire d'Informatique (LIX)\\ Palaiseau\\ France}

\maketitle
 
\begin{abstract}
  Apart from few exceptions, the mathematical runtime analysis of evolutionary algorithms is mostly concerned with expected runtimes. In this work, we argue that stochastic domination is a notion that should be used more frequently in this area. Stochastic domination allows to formulate much more informative performance guarantees, it allows to decouple the algorithm analysis into the true algorithmic part of detecting a domination statement and the probability-theoretical part of deriving the desired probabilistic guarantees from this statement, and it helps finding simpler and more natural proofs. 
  
  As particular results, we prove a fitness level theorem which shows that the runtime is dominated by a sum of independent geometric random variables, we prove the first tail bounds for several classic runtime problems, and we give a short and natural proof for Witt's result that the runtime of any $(\mu,p)$ mutation-based algorithm on any function with unique optimum is subdominated by the runtime of a variant of the \oea on the \onemax function.
  As side-products, we determine the fastest unbiased (1+1) algorithm for the \leadingones benchmark problem, both in the general case and when restricted to static mutation operators, and we prove a Chernoff-type tail bound for sums of independent coupon collector distributions.
\end{abstract}


\section{Introduction}

The analysis of evolutionary algorithms via mathematical means is an established part of evolutionary computation research. The subarea of \emph{runtime analysis} aims at giving proven performance guarantees on the time an evolutionary algorithm takes to find optimal or near-optimal solutions. Traditionally, this area produces estimates for the expected runtime, which are occasionally augmented by tail bounds. A justification for the restriction to expectations was that already for very simply evolutionary algorithms and optimization problems, the stochastic processes arising from running this algorithm on this problem are so complicated that any more detailed analysis is infeasible. See the analysis how the $(1+1)$ evolutionary algorithm optimizes linear functions $f \colon \{0,1\}^n \to \R; x \mapsto a_1 x_1 + \dots + a_n x_n$ in~\cite{DrosteJW02} for an example.

In this work, we shall argue that the restriction to expectations is less justified and propose stochastic domination as an alternative. It is clear that the precise distribution of the runtime of an evolutionary algorithms often is out of reach (see Section~\ref{sec:precise} for an overview on what runtime distributions are known exactly). Finding the precise distribution is maybe not even an interesting target because most likely already the result will be too complicated to be useful. What would be very useful is a possibly not absolutely tight upper bound-type statement that concerns the whole distribution of the runtime. 

One way to formalize such statement is via the notion of \emph{stochastic domination}. A real-valued random variable $Y$ stochastically dominates another one $X$ if and only if for each $\lambda \in \R$, we have 
\[\Pr[X \le \lambda] \ge \Pr[Y \le \lambda].\] 
If $X$ and $Y$ describe the runtimes of two algorithms $A$ and $B$, then this domination statement is a very strong way of saying that algorithm $A$ is at least as fast as $B$. See Section~\ref{secdom} for a more detailed discussion of the implication of such a domination statement.

\subsection{Our Results}

In this work, we shall give three main arguments for a more frequent use of domination arguments in runtime analysis and support these with several new runtime results.

\paragraph{Stochastic domination is often easy to show.} Surprisingly, despite being a much stronger type of assertion, stochastic domination statements are often easy to obtain. The reason is that many of the classic proofs implicitly contain all necessary information, they only fail to formulate the result as a domination statement. 

As an example, we prove in Section~\ref{secfitness} a natural domination version of the classic fitness level method. In analogy to the classic result, which translates pessimistic improvement probabilities $p_1, \ldots, p_{m-1}$ into an expected runtime estimate $E[T] \le \sum_{i=1}^{m-1} \frac 1 {p_i}$, we show that under the same assumptions the runtime is dominated by the sum of independent geometric random variables with success probabilities $p_1, \ldots, p_{m-1}$. This statement implies the classic result, but also implies tail bound for the runtime via Chernoff bounds for geometric random variables. We note that, while our extension of the fitness level theorem is very natural, the proof is not totally obvious, which might explain why previous works were restricted to the expectation version. 

\paragraph{Stochastic domination allows to separate the core algorithm analysis and the probability theoretic derivation of probabilistic runtime statements.} The reason why stochastic domination statements are often easy to obtain is that they are close to the actions of the algorithms. When we are waiting for the algorithm to succeed in performing a particular action, then it is a geometric distribution that describes this waiting time. To make such a statement precise for a given problem, we need to understand how the algorithm achieves the particular goal. This requires a good understanding of the algorithm and the problem as well as some elementary probability theory and discrete mathematics, but usually no greater expertise in probability theory. We give in Sections~\ref{secfitnessappl} and~\ref{secbeyondfitness} several examples, mostly using the fitness level method, but also concerning the single-source shortest paths problem, where the fitness level method is not suitable to give the best known results.

Once we have a domination statement formulated, e.g., that the runtime is dominated by a sum of independent geometric distributions, then deeper probability theoretic arguments like Chernoff-type tail bounds come into play. This part of the analysis is independent of the algorithm and only relies on the domination statement. Exemplarily, we derive tail bounds for the runtime of the \oplea on \onemax, the \mpoea on \leadingones, the \oea on jump functions, the \oea for sorting, the multi-criteria \oea for the single-source shortest path problem, and some more results.

That the two stages of the analysis are of a different nature is also visible in the history of runtime analysis. As discussed above, many classic proofs essentially contains all ingredients to formulate a domination statement. However, the mathematical tools to analyze sums of independent geometric random variables were developed much later (and this development is still ongoing). 

This historical note also shows that from the viewpoint of research organization, it would have been profitable if previous works would have been formulated in terms of domination statements. This might have spurred a faster development of suitable Chernoff bounds and, in any case, it would have made it easier to apply the recently found Chernoff bounds (see Theorems~\ref{tprobcgeomungleich} to~\ref{tprobgeomharmonic}) to these algorithmic problems.

\paragraph{Stochastic domination often leads to more natural and shorter proofs.} To demonstrate this in an application substantially different from the previous ones, we regard the classic lower-bound result which, in simple words, states that solving \onemax via the \oea is the fastest optimization process among many mutation-based algorithms and problems with unique optimum. This statement, of course, should be formulated via stochastic domination and this has indeed been done in previous work. However, as we shall argue in Section~\ref{seclower}, we also have the statement that when comparing the two processes the distance of the current-best solution from the optimum for the general function dominates this distance for the \onemax function. This natural statement immediately implies the domination relation between the runtimes. We make this precise for the current-strongest \onemax-is-easiest result~\cite{Witt13}. This will shorten the previous, two-and-a-half pages long  complicated proof to an intuitive proof of less than a page. 

\paragraph{Side-results.} We use our discussion of previous results on precisely known distributions of runtimes of evolutionary algorithms to (i)~give a formal proof for the distribution result on the runtime of the \oea on the \leadingones benchmark function stated without proof in~\cite{DoerrJWZ13}, (ii)~to extend this result to a wide class of (1+1) type algorithms and to time-to-target runtimes, and (iii)~to determine the best possible unbiased (1+1) algorithm for this benchmark function, both in the general case and with the restriction to static mutation operators (Section~\ref{sec:deflo}). 

We also use our discussion of Chernoff bounds for sums of geometric random variables to show a Chernoff bound for sums of coupon collector random variables (Section~\ref{sec:chernoffgeom}).

\subsection*{Related work} 

The use of stochastic domination is not totally new in the theory of evolutionary computation, however, the results so far appear rather sporadic than systematic. 

Possibly the first to use the notion of stochastic domination was Droste, who in~\cite{Droste03,Droste04} employed it to make precise an argument often used in an informal manner, namely that some artificial random process is not faster than the process describing a run of the algorithm under investigation. Using domination as \emph{tool in proofs} has been done subsequently in various applications, among others, binary particle swarm (PSO) algorithms~\cite{SudholtW08}, ant-colony optimization~\cite{SudholtT12}, evolutionary multi-objective optimization~\cite{BringmannF13}, optimization of monotonic functions~\cite{ColinDF14,Lengler18}, non-elitist algorithms~\cite{RoweS14}, fixed budget analyses~\cite{LenglerS15}, the simple GA~\cite{OlivetoW15}, estimation-of-distribution algorithms~\cite{FriedrichKKS17,KrejcaW17}, genetic programming~\cite{DoerrKLL17}, island models~\cite{DoerrFFFKS17}, and dynamic and noisy optimization~\cite{Dang-NhuDDIN18}.

The first time that stochastic domination was used to \emph{formulate results} (though without explicitly calling it stochastic domination) was in~\cite{BorisovskyE08}, where several result where shown of the type that when one mutation operator is better than another one (in a suitable domination sense), then the algorithm using the first one always has at least as good solutions as the second (in the sense of stochastic domination). In~\cite{Witt13}, the result is shown that the runtime of an arbitrary $(\mu,p)$ mutation-based algorithm on any objective function with unique optimum dominates the runtime of the \oea with mutation rate $p$ and best-of-$\mu$ initialization on the \onemax function (see Section~\ref{seclower} for more details on this result).

What comes closest to this work is the paper~\cite{ZhouLLH12}, which also tries to establish runtime analysis beyond expectations in a formalized manner. The notion proposed in~\cite{ZhouLLH12}, called \emph{probable computational time} $L(\delta)$, is the smallest time $T$ such that the algorithm under investigation within the first $T$ fitness evaluations finds an optimal solution with probability at least $1-\delta$. In some sense, this notion can be seen as an inverse of the classic tail bound language, which makes statements of the type that the probability that the algorithm needs more than $T$ iterations, is at most some function in $T$, which is often has an exponential tail. To prove results on the probable computational time, the authors implicitly use (but do not prove) the result that the fitness level method gives a stochastic domination by a sum of independent geometric random variables (which we prove in Section~\ref{secfitness}). When using this result to obtain bounds on the probable computational time, they suffer from the fact that at that time, good tail bounds for sums of geometric random variables with different success probabilities where not yet available (these appeared only in~\cite{Witt14,DoerrD15tight}). For this reason, they had to use a self-made tail bound, which unfortunately gives a non-trivial tail probability only for values higher than twice the expectation. With this limitation, tail bounds for the runtimes of the algorithms RLS, \oea, \mpoea, MMAS$^*$ and binary PSO on the optimization problems \onemax and \leadingones are proven. 

In~\cite{DoerrL17ecj}, the complexity theoretic notion of $p$-Monte Carlo black-box complexity was introduced, which serves at providing more information than what expected runtimes convey for lower bounds on runtimes.


\section{Exact Distributions}\label{sec:precise}

We said above that stochastic domination can be used to formulate runtime bounds in a distributional sense, which can give a very complete picture but without having to determine the exact runtime distribution. Before going into details, we shall summarize in this section the few situations in which it has been possible to describe the runtime distribution precisely. Not surprisingly, these all regard very simple randomized search heuristics. 

This section contains a few new results, e.g., the determination of the fastest unbiased (1+1) algorithm for the \leadingones function for static and dynamic mutation operators.

\subsection{Random Search}\label{sec:geomdef}\label{sec:randomsearch}

The \emph{random search heuristics} repeatedly samples solutions chosen uniformly at random until some termination criterion is met. Therefore, the time to sample a certain solution (or one of a given subset) follows a \emph{geometric distribution}. We say that a random variable $X$ is geometrically distributed with parameter $0 < p \le 1$ and write $X \sim \Geom(p)$ if  \[\Pr[X = k] = (1-p)^{k-1} p\] for all $k \in \N$. Note that this implies $\Pr[X \ge k] = (1-p)^{k-1}$ and $E[X] = \frac 1p$.

\begin{lemma}\label{lem:randomsearch}
  Consider some search or optimization problem with finite search space $\calS$. Let $\calT$ be a non-empty subset of $\calS$. Let $T$ be the random variable describing the first time random search samples a solution from $\calT$. Then 
  \[T \sim \Geom\left(\frac{|\calT|}{|\calS|}\right).\]
  In particular, the time $T$ to find a unique optimum in a problem with bit-string representation of length $n$ is $T \sim \Geom(2^{-n})$.
\end{lemma}

\subsection{Randomized Local Search}\label{sec:onemaxdef}\label{sec:rls}

The \emph{randomized local search} (RLS) heuristic starts with a random search point and then repeats sampling a neighbor of the current solution and replacing the current solution with this offspring if it is at least as good. In the case of the bit-string representation of length $n$, neighbors are bit-strings which differ in exactly one bit. Hence this heuristic repeats flipping a single random bit and continues with the new solution if it is not worse than the parent. This is exactly Algorithm~\ref{alg:oea} with the $1$-bit mutation operator.
\begin{algorithm2e}%
	Choose $x \in \{0,1\}^n$ uniformly at random\;
  \For{$t=1,2,3,\ldots$}{
    $y \assign \mutate(x)$\;
    \lIf{$f(y)\geq f(x)$}{$x \assign y$}
  }
\caption{The family of (1+1) search heuristics for maximizing a given objective function~$f\colon\{0,1\}^n\to\mathbb{R}$. When the mutation operator returns a random Hamming neighbor of~$x$, that is, flips a random bit in~$x$, then this algorithm is the RLS heuristic. When the mutation operator flips each bit independently with probability $p$, then we obtain the \oea with mutation rate $p$. The classic \oea uses a mutation rate of $p=\frac 1n$.}
\label{alg:oea}
\end{algorithm2e}

When used to optimize the classic benchmark function $\onemax \colon \{0,1\}^n \to \R$ defined by \[\onemax(x) = \sum_{i=1}^n x_i\] or any other \emph{strictly monotonic} function $f \colon \{0,1\}^n \to \R$ (where flipping a $0$-bit into a $1$ always increases the fitness), the resulting optimization process is equivalent to a coupon collector process in which each type of coupon is present initially with probability $\frac 12$, see~\cite{DoerrD16}. The runtime of the coupon collector process with $n-k$ of the $n$ coupons initially present is $\sum_{i=1}^k \Geom(\frac in)$, where this sum is understood as sum of independent distributions. Consequently, the runtime of RLS on strictly monotonic functions can be described as follows, where we write $X \sim \Bin(n,p)$ to indicate that $X$ follows a \emph{binomial distribution} with parameters $n$ and $p \in [0,1]$, that is, we have $\Pr[X = i] = \binom{n}{i} p^i (1-p)^{n-i}$ for all $i \in [0..n]$.

\begin{lemma}
  Let $f \colon \{0,1\}^n \to \R$ be any strictly monotonic function. Let $T$ be the number of iterations\footnote{In this work, we shall call the runtime of an iterative algorithm the number of iterations taken until the optimum is generated for the first time. In comparison with the classic definition, the number of fitness evaluations until the optimum is evaluated, we thus do not count the initial search points and, in the case that per iteration $\lambda > 1$ individuals are generated, we do not count the runtime of the iteration as $\lambda$, but as one.} taken by RLS to sample the (unique) optimum of~$f$. Then \[T \sim \sum_{i=0}^{n-1} 1_{X \le i} \cdot \Geom(\tfrac{n-i}{n}),\]
  where $X \sim \Bin(n,\frac 12)$ and we assume that $X$ and the arising geometric distributions are independent.
\end{lemma}

\subsection{The LeadingOnes Benchmark Problem}\label{sec:deflo}

For a surprisingly large class of algorithms the precise runtime on the \leadingones benchmark problem can be determined. The function $\LO \colon \{0,1\}^n \to \R$ is defined by \[\LO(x) = \min\{i \in [0..n] \mid \forall j \in [1..n] \cap \R_{\le i} : x_j = 1\}\] for all $x \in \{0,1\}^n$. 

In parallel independent work, the precise expected runtime of the \oea on the \leadingones benchmark function was determined in~\cite{BottcherDN10,Sudholt13} (note that~\cite{Sudholt13} is the journal version of a work that appeared at the same conference as~\cite{BottcherDN10}). In terms of results, the two works are very similar. The work~\cite{Sudholt13} is also regards the \oea with Best-of-$\mu$ initialization and shows that a lower-order runtime gain can be obtained from taking as initial solution the best of, say, $n$ random search points. The work~\cite{BottcherDN10} also shows that the often recommended mutation rate of $p=1/n$ is not optimal. A runtime smaller by 16\% can be obtained from taking $p = 1.59 / n$ and another 12\% can be gained by using a fitness-dependent mutation rate. 

In terms of methods, the two works are substantially different. An advantage of the analysis in~\cite{BottcherDN10} is that it can easily be reformulated to give the precise distribution of this runtime. This was first observed in~\cite{DoerrJWZ13}, but used several times subsequently. The arguments in~\cite{BottcherDN10} are not specific to the \oea, but are valid for any other (1+1) algorithm as defined in Algorithm~\ref{alg:oea}. Since already the observation in~\cite{DoerrJWZ13} was not formally proven, we now quickly formulate the general result and prove it.

\begin{theorem}\label{thmLO}
  Consider the run of a (1+1) algorithm optimizing the \LO function. Let $T$ be the first time the optimum is generated. Then 
  \[T \sim \sum_{i=0}^{n-1} X_i \cdot \Geom(q_i),\]
  where $X_0, X_1, \dots, X_{n-1}$ are uniformly distributed binary random variables, the $X_i$ and $\Geom(q_i)$ are mutually independent, and for all $i \in [0..n-1]$ we denote by $q_i$ the probability that the mutation operator generates from a search point of fitness exactly $i$ a strictly better search point. Consequently, $E[T] = \frac 12 \sum_{i=0}^{n-1} \frac 1 {q_i}$, where we read $\frac 1 {p_i} = \infty$ when $p_i = 0$.
\end{theorem}

\begin{proof}
  For all $i \in [0..n]$, denote by $T_i^0$ the runtime of the algorithm when starting with a random search point of fitness exactly $i$. Note that if $x$ is such a random search point, then $x_j = 1$ for $j \in [1..i]$, $x_{i+1} = 0$, and for $j \in [i+2..n]$ the $x_j$ are independent random variables uniformly distributed in $\{0,1\}$. Let $T_i^{\rand}$ denote the runtime of the algorithm when starting with a random search point $x$ of fitness at least $i$, that is, we have $x_j = 1$ for $j \in [1..i]$ and for $j \in [i+1..n]$ the $x_j$ are independent random variables uniformly distributed in $\{0,1\}$. 
  
  Trivially, we have $T_n^0 = T_n^{\rand} = 0$. The main technical insight is that due to the random initialization in the algorithm, we have 
  \begin{equation}
  T_i^0 = \Geom(q_i) + T_{i+1}^{\rand} \label{eq:insight}
  \end{equation}
  for $i < n$. Here $\Geom(q_i)$ describes the waiting time for an improvement when having a search point of fitness exactly $i$. At the moment when this improvement happens, the bits with index $i+2$ or higher are still independent and uniformly distributed binary random variables, since their value cannot have had any influence on the run of the algorithm. Formally speaking, we have $(x_{i+2}, \dots, x_n) = (x^0_{i+2}, \dots, x^0_n) \oplus (y_{i+2}, \dots, y_n)$, where $x^0$ denotes the random initial search point and $y = (y_{i+2}, \dots, y_n)$ describes which bits have changed compared to the initial bit-string. Note that $y$ is independent of $(x^0_{i+2}, \dots, x^0_n)$. Hence since $(x^0_{i+2}, \dots, x^0_n)$ is a random binary string, also $(x_{i+2}, \dots, x_n) = (x^0_{i+2}, \dots, x^0_n) \oplus (y_{i+2}, \dots, y_n)$ is a random binary string.
  
  With~\eqref{eq:insight} we easily compute  
  \begin{align*}
  T_i^{\rand} &= X_i T_i^0 + (1-X_i) T_{i+1}^{\rand}\\
  &= X_i (\Geom(q_i) + T_{i+1}^{\rand}) + (1-X_i) T_{i+1}^{\rand}\\
  &= X_i \Geom(q_i) + T_{i+1}^{\rand},
  \end{align*}
  where $X_i$ is a uniformly distributed binary random variable independent of any randomness used in the other distributions. From the latter equation, a simple induction proves the claim. Note that the runtime of the algorithm is $T_0^{\rand}$.
\end{proof}

As observed in~\cite{DoerrJWZ13} and exploited for a fixed-budget analysis, an analogous result is valid for the time $T^{\to a}$ to first reach a search point of fitness at least $a \in [0..n]$, namely $T^{\to a} \sim \sum_{i=0}^{a-1} X_i \cdot \Geom(q_i)$; note that in Lemma~5 of~\cite{DoerrJWZ13}, the range of the sum starts at $1$, but this clearly is a typo. Since this extension to time-to-target runtimes can be important (see~\cite{CarvalhoD17}), we quickly make this point precise even if in the remainder of this work we restrict ourselves to the traditional runtime analysis target of determining the time to find the optimum.
\begin{corollary}
  In the notation of Theorem~\ref{thmLO}, the time $T^{\to a}$ the algorithm takes to find a solution of fitness at least $a$ is 
  \[T^{\to a} \sim \sum_{i=0}^{a-1} X_i \cdot \Geom(q_i)\]
  for all $a \in [0..n]$.
\end{corollary}

\begin{proof}
  Let $a \in [0..n]$. Denote by $T_i^0$, $T_i^{\rand}$ the analogues of the times defined in the previous proof, but for the target of reaching a fitness of at least $a$. Then $T_a^0 = T_a^{\rand} = 0$ and for $i < a$ the same relations hold as above. This proves the claim.
\end{proof}

Theorem~\ref{thmLO} allows to determine precisely the expected runtimes of many (1+1) algorithms. More importantly, it also allows to determine optimal parameters for the mutation operator.
\begin{itemize}
\item For the classic RLS heuristic always flipping a single random bit, we have $q_i = \frac 1n$ for all $i \in [0..n-1]$. Consequently, the expected runtime is $E[T] = 0.5 n^2$.  
As we shall show below, this is also the best static unbiased mutation operator for the \leadingones problem. 
\item Motivated by an analysis how evolutionary algorithms can solve problems with unknown solution length (that is, how many bits are actually relevant for the problem), in~\cite[Lemma~3.4]{DoerrDK17} an extension of the RLS heuristic was analyzed which, for a given sequence $p_1, p_2, \dots, p_n \in [0,1]$ with $\sum_{i=1}^n p_i \le 1$, as $1$-bit mutation flips the $i$-th bit with probability exactly $p_i$. For this mutation operator, an expected runtime of $\tfrac 12 \sum_{i=1}^n \frac 1 {p_i}$ was shown for instances of fixed length $n$. Hence for known solution length, there is no gain from flipping bits with position-dependent probabilities. This result is now an immediate consequence of Theorem~\ref{thmLO}. 	 
\item For the variant of RLS which randomly mixes the 1-bit and 2-bit flip operator, that is, with probability $P$ one random bit is flipped and with probability $1-P$ two (different) random bits are flipped, we have $q_i = \frac Pn + 2 (1-P) \frac{n-i-1}{n(n-1)}$. From this, a slightly tedious calculation gives an expected runtime of $E[T] = \frac{1}{4(1-p)} \ln(\frac{2-p}{p}) n^2 +o(n^2)$, see~\cite{LissovoiOW17} and note that the fact that the authors consider 2-bit flips with repetition (that is, with probability $1/n$ the same bit is flipped twice resulting in a copy of the parent) has no influence on the result apart from lower-order terms.
\item  When flipping $k$ random bits in a search point with fitness $i$, the probability $q(n,k,i)$ of obtaining a strictly better solution is $q(n,k,i) = \frac{k (n-i-1) \dots (n-i-k+1)}{n (n-1) \dots (n-k+1)}$. Consequently, $q(n,k,i) \le q(n,k+1,i)$ if and only if $i \le \frac{n-k}{k+1}$. This gives that the best (that is, largest) value for $q_i$ is obtained from flipping $k(n,i) := \lfloor \frac {n} {i+1} \rfloor$ bits when the current fitness is $i$.\footnote{In an earlier version of this work, we said that $\lfloor \frac{n+1}{i+1} \rfloor$ would be the optimal number of bits to flip in a search point with fitness $i$. Apart from the trivial case $i=0$, where the previous formulation gives $n+1$ instead of the desired value of $n$, both versions are correct. For all $i >0$ such that two definitions differ, both numbers of bits to be flipped give the same probability for a fitness improvement. Thanks to Carola Doerr and Markus Wagner for pointing me to the new formula from their forthcoming work~\cite{DoerrW18}.} This fitness-dependent choice of the mutation operator gives a runtime of approximately $0.39 n^2$. Since any unary unbiased mutation operator (see~\cite{LehreW12}) is the convex combination of $k$-bit flip operators (see~\cite{DoerrDY16gecco} or~\cite{DoerrKLW13tcs}), this algorithm also is the fastest unbiased (1+1) algorithm for \leadingones.
\item For the \oea, that is, when the mutation operator is standard-bit mutation with mutation rate $p$, we have $q_i = q_i(p) = (1-p)^{i} p$ for all $i \in [0..n-1]$. This yields an expected runtime of $E[T] = \frac{1}{2p^2}((1-p)^{1-n} - (1-p))$. For the standard choice $p=\frac 1n$, this is $E[T] = \frac 12 (e-1-o(1)) n^2 \approx 0.86n^2$. A better (and the asymptotically best among all static rates) expected runtime is obtained from using $p \approx 1.59/n$, giving $E[T] \approx 0.77n^2$. The best (fitness-dependent) mutation rate is using $p_i = \frac{1}{i+1}$ when the current fitness is $i$, giving a runtime of $E[T] = \frac e4 n^2 \pm O(n) \approx 0.77n^2$. All these results are from~\cite{BottcherDN10}. Apparently unaware of this work, the expected runtime in the case $p=\frac 1n$ was again determined in~\cite{DuASFY17}. 
\item Since the heavy-tailed mutation operator proposed in~\cite{DoerrLMN17} applies standard-bit mutation with a randomly chosen mutation rate, the $q_i$ are convex combinations of the $q_i(p)$ computed above. This again determines the expected runtime of the \oea with this mutation operator, however, a simple closed formula as for standard-bit mutation with fixed rate is not known.
\item When allowing position-dependent mutation rates for the \oea, that is, the $i$-th bit is flipped with probability $p_i$ independently for all $i \in [1..n]$, then an expected runtime of $E[T] = \frac 12 \sum_{i=1}^n (p_i \prod_{j=1}^{i-1} (1-p_j))^{-1}$ was shown in~\cite{DoerrDK17}. This is minimized by taking $p_i = \frac 1i$ for all $i$, which yields $E[T] = 0.5 n^2$. 
\end{itemize}

We now prove the results promised above that the $1$-bit mutation operator as used by RLS is the (unique) best unbiased static mutation operator for the (1+1) algorithm scheme. We say that a (1+1) algorithm uses a static mutation operator if for all $x \in \{0,1\}^n$ the distribution of the offspring generated from $x$ is the same throughout the run of the algorithm. 

\begin{theorem}
  For any (1+1) algorithm using static unbiased mutation, the expected runtime on the \leadingones function is at least $0.5 n^2$. Moreover, RLS is the only such algorithm satisfying this bound.
\end{theorem}

\begin{proof}
  Consider a static unbiased mutation operator. As above, there are $r_0, r_1, \dots, r_{n} \in [0,1]$ such that $\sum_{k=0}^n r_k = 1$ and such that the mutation operator (unchanged throughout the run of the algorithm) can be written as ``sample $k \in [0..n]$ with probability $r_k$ and then flip exactly $k$ random bits''. 
  
  Let again $q(n,k,i) = \frac{k (n-i-1) \dots (n-i-k+1)}{n (n-1) \dots (n-k+1)}$ be the probability of obtaining a strictly better solution from a search point of fitness $i$ by flipping $k$ bits. We apply Theorem~\ref{thmLO} with $q_i = \sum_{k=0}^{n-1} r_k q(n,k,i)$ and obtain an expected optimization time of $E[T] = \frac 12 \sum_{i=0}^{n-1} \frac 1 {q_i}$. We show that this expression, viewed as function of the $r_k$, has a unique minimum for $r_1 = 1$ (and hence $r_k = 0$ for all $k \neq 1$). 
  
  It is clear that flipping zero bits can never lead to an improvement. Hence moving any positive mass on $r_0$ to $r_1$ would strictly improve the algorithm. Without loss of generality, we thus assume in the following that $r_0 = 0$.
  
  We first show that, regardless of the remaining $r_k$, we have $\sum_{i=0}^{n-1} q_i = 1$. This is, naturally, equivalent to saying that \begin{equation}
  \sum_{i=0}^{n-1} q(n,k,i) = 1 \label{eq:eins}
\end{equation}
for all $k \in [1..n]$. Since we have a precise expression for the $q(n,k,i)$, equation~\eqref{eq:eins} in principle could be shown via Faulhaber's formula~\cite{Faulhaber31}. Fortunately, a much simple probabilistic argument can be applied. Let $A_i$ be the event that in an application of the $k$-bit mutation operator the $i$-th bit is flipped and that no bit $j$ with $j < i$ is flipped. Then $\Pr[A_i] = q(n,k,i-1)$. Trivially, the $A_i$ are disjoint events that cover the whole probability space. Hence 
\[1 = \Pr[A_1 \cup \dots \cup A_{n}] = \sum_{i=1}^{n} \Pr[A_i] = \sum_{i=0}^{n-1} q(n,k,i)\] 
for all $k \in [0..n-1]$ and thus $\sum_{i=0}^{n-1} q_i = 1$. 
  
  Next, we recall that among all $q_0, \dots, q_{n-1}$ with $\sum_{i=0}^{n-1} q_i = 1$, the sum $\sum_{i=0}^{n-1} \frac{1}{q_i}$ of the reciprocals is minimal if and only if the $q_i$ are all equal. Since $i \mapsto q(n,k,i)$ is non-increasing, so is $i \mapsto q_i$, and the only way to have the $q_i$ all equal is that for all $k \in [1..n]$ the function $i \mapsto r_k q(n,k,i)$ is constant. Since only for $k=1$ the function $i \mapsto q(n,k,i)$ is constant, we have $r_k = 0$ for all $k \neq 1$, and thus $r_1 = 1$ as claimed. 
\end{proof}

\subsection{Conclusion on Exact Runtime Distributions}

The results on \leadingones, in particular, how the knowledge of the precise distribution easily yields optimal parameter values, show that understanding the precise runtime distribution would be highly desirable. 

However, it has to be noticed that the results presented above are limited to very restricted settings and it is highly unlikely that they can be extended significantly. For example, that the exact runtime analysis of randomized local search on \onemax can be extended to other algorithms appears hard to believe given that for the runtime of the \oea on \onemax, despite significant efforts~\cite{Rudolph97,DoerrFW10,DoerrFW11,Sudholt13,HwangPRTC18}, not even an exact value for the expectation is known. Similarly, the results for \leadingones depend crucially on the facts that only (1+1) algorithms are regarded and that the initial individual is random. For a different initialization, none of the results would hold and also analogous results determining the exact runtime distribution are not in sight. 

For this reason, the approach discussed in the following is promising: To not try to find the exact distribution, but to ask for distributions which are upper bounds in a very strong sense, namely in the sense of stochastic domination. With this notion we will then, e.g., be able to say that the runtime distribution of the \oea on \onemax is dominated by $\sum_{i=1}^n \Geom(\frac i{en})$ or that the runtime distribution for an arbitrary (1+1) algorithm on $\leadingones$ is dominated by $\sum_{i=0}^{n-1} \Geom(q_i)$ regardless of how the initial solution is chosen. 

We note without further detail that another way of weakening the aim of an exact runtime distribution is to give a limiting distribution only. This has been done for the runtime of RLS and \oea on \onemax and the \needle function~\cite{GarnierKS99} and on the \leadingones function~\cite{Ladret05}. For example, it is shown in~\cite{Ladret05} that the runtime $T_n$ of the \oea with mutation rate $c/n$ on the $n$-dimensional \leadingones function has the property that 
\[\frac{T_n - \frac{e^c-1}{2c^2}}{n^{3/2}}\]
converges in distribution to a centered Gaussian random variable with variance 
\[\sigma^2 = \frac{3(e^{2c}-1)}{8c^3}.\]
The results on the runtime of the \oea on \onemax and \leadingones have been reproven with different methods and made significantly more precise in~\cite{HwangPRTC18}.

\section{Stochastic Domination}\label{secdom}

In this section, we recall the definition of \emph{stochastic domination} and collect a few known properties of this notion. For an extensive treatment of various forms of stochastic orders, we refer to~\cite{MullerS02}.

\subsection{Definition of Stochastic Domination}

Stochastic domination is usually defined as follows.

\begin{definition}[Stochastic domination]
  Let $X$ and $Y$ be random variables not necessarily defined on the same probability space. We say that $Y$ stochastically dominates $X$, written as $X \preceq Y$, if for all $\lambda \in \R$ we have $\Pr[X \le \lambda] \ge \Pr[Y \le \lambda]$.
\end{definition}

If $Y$ dominates $X$, then the cumulative distribution function of $Y$ is point-wise not larger than the one of $X$. The definition of domination is equivalent to ``$\forall \lambda \in \R : \Pr[X \ge \lambda] \le \Pr[Y \ge \lambda]$'', which shows more clearly why we feel that $Y$ is at least as large as $X$. 

Concerning nomenclature, we remark that some research communities in addition require that the inequality is strict for at least one value of $\lambda$. Hence, intuitively speaking, $Y$ is strictly larger than $X$. From the mathematical perspective, this appears not very practical, and from the computer science perspective it is not clear what can be gained from this alternative definition. Consequently, our definition above is  more common in computer science (though, e.g., \cite{ZhouLLH12} also use the alternative definition). 

One advantage of comparing two distributions via the notion of domination is that this makes a statement over the whole domain of the distributions, including ``rare events'' on the tails. If the runtime $T_A$ of some algorithm $A$ is dominated by the runtime $T_B$ of algorithm $B$, then not only  $A$ is better than $B$ in average, but also exceptionally large runtimes occur less frequent when using $A$.

A second advantage is that domination is invariant under monotonic re-scaling. Imagine that running an algorithm for time $t$ incurs some cost $c(t)$. We may clearly assume that $c$ is monotonically increasing, that is, that $t_1 \le t_2$ implies $c(t_1) \le c(t_2)$. Then $T_A \preceq T_B$ implies $c(T_A) \preceq c(T_B)$. Hence changing the cost measure does not change our feeling that algorithm $A$ is better than $B$. Note that this is different for expectations. We may well have $E[T_A] < E[T_B]$, but $E[c(T_A)] > E[c(T_B)]$. 

\subsection{Properties of Stochastic Domination}

We collect a few useful properties of stochastic domination. We start with three elementary observations.

\begin{lemma}\label{lprobdomexp}
  If $X \preceq Y$, then $E[X] \le E[Y]$.
\end{lemma}

\begin{samepage}
\begin{lemma}
  The following two conditions are equivalent.\nopagebreak
  \begin{enumerate}
	  \item $X \preceq Y$.
	  \item For all monotonically increasing functions $f \colon \R \to \R$, we have \[E[f(X)] \le E[f(Y)].\]
  \end{enumerate}
\end{lemma}
\end{samepage}

\begin{lemma}\label{Xlprobtrivialdom}
Let $X$ and $Y$ be random variables.
\begin{enumerate}
  \item If $X$ and $Y$ are defined on the same probability space and $X \le Y$, then $X \preceq Y$.
  \item If $X$ and $Y$ are identically distributed, then $X \preceq Y$.
\end{enumerate}
\end{lemma}

The following non-trivial lemma will be needed in our proof of the extended fitness level theorem. It was first proven a in slightly weaker form in~\cite{DoerrHK12} and~\cite{Doerr11bookchapter}. The current version is from~\cite{Doerr18bookchapter}.

\begin{lemma}\label{lprobdomind}
  Let $X_1,\dots,X_n$ be arbitrary discrete random variables. Let $X^*_1,\dots,X^*_n$ be discrete random variables that are mutually independent.  Assume that for all $i \in [1..n]$ and all $x_1,\dots, x_{i-1}$ with $\Pr[X_1=x_1,\dots,X_{i-1}=x_{i-1}]>0$, we have \[\Pr[X_i\ge k \mid X_1=x_1,\dots,X_{i-1}=x_{i-1}] \le \Pr[X_i^* \ge k]\] for all $k \in \R$, that is, $X_i^*$ dominates $(X_i \mid X_1=x_1,\dots,X_{i-1}=x_{i-1})$. Then \[\sum_{i=1}^n X_i \preceq \sum_{i=1}^n X_i^*.\]
\end{lemma}

The following result is again elementary. For discrete random variables, it is a special case of Lemma~\ref{lprobdomind}.

\begin{lemma}\label{lprobdomsum}
  Let $X_1, \dots, X_n$ be independent random variables defined over some common probability space. Let $Y_1, \dots, Y_n$ be independent random variables defined over a possibly different probability space. If $X_i \preceq Y_i$ for all $i \in [1..n]$, then \[\sum_{i=1}^n X_i \preceq \sum_{i=1}^n Y_i.\]
\end{lemma}

\subsection{Coupling}

Stochastic domination is tightly connected to \emph{coupling}. Coupling is an analysis technique that consists of suitably defining two unrelated random variables over the same probability space to ease comparing them. Let $X$ and $Y$ be two random variables not necessarily defined over the same probability space. We say that $(\tilde X, \tilde Y)$ is a \emph{coupling} of $(X,Y)$ if $\tilde X$ and $\tilde Y$ are defined over a common probability space and if $X$ and $X'$ as well as $Y$ and $Y'$ are identically distributed. This definition itself is very weak. $(X,Y)$ have many couplings and most of them are not interesting. So the art of coupling as a proof and analysis technique is to find a coupling of $(X,Y)$ that allows to derive some useful information. This is often problem-specific, however, also the following general result in known. It in particular allows to couple dominating random variables. We shall use it for this purpose in Section~\ref{seclower}.
\begin{theorem}\label{tprobdomcou}
  Let $X$ and $Y$ be random variables. Then the following two statements are equivalent.
  \begin{enumerate}
	\item $X \preceq Y$.
	\item There is a coupling $(\tilde X, \tilde Y)$ of $(X,Y)$ such that $\tilde X \le \tilde Y$.
\end{enumerate}
\end{theorem}

Without going into further detail, we note that coupling is not restricted to real-values random variables. In the analysis of population-based evolutionary algorithms, a powerful strategy to prove lower bounds is to couple the true population of the algorithm with the population of an artificial process without selection and by this overcome the difficult dependencies introduced by the variation-selection cycle of the algorithm. This was first done in~\cite{Witt06} and~\cite{Witt08} for the analysis of the \mpoea and an elitist steady-state GA. This technique then found applications for memetic algorithms~\cite{Sudholt09}, aging-mechanisms~\cite{JansenZ11tcs}, non-elitist algorithms~\cite{LehreY12}, multi-objective evolutionary algorithms~\cite{DoerrKV13}, and the \mplea~\cite{AntipovDFH18}.

\section{Domination-based Fitness Level Method}\label{secfitness}

In this section, we prove a version of the fitness level theorem that gives domination statements and we apply it to several classic problems. 

The fitness level method, invented by Wegener~\cite{Wegener01}, is one of the most successful early analysis methods in the theory of evolutionary computation. It builds on the idea of partitioning the search space into levels $A_i$, $i = 1, \dots, m$, which contain search points of strictly increasing fitness (that is, for all $i \in [1..m-1]$, $x \in A_i$, and $y \in A_{i+1}$, the fitness of $y$ is better than the one of $x$). 

We then try to show a lower bound $p_i$ for the probability that, given that the best-so-far search point is in $A_i$, we generate in one iteration a search point in a higher level (due to this reference to the best-so-far search point, the fitness level theorem is mostly used for elitist algorithms). From this data, the fitness level theorem gives the estimate 
\[E[T] \le \sum_{i=1}^{m-1} \frac 1 {p_i}\] 
for the time $T$ (that is, the number of iterations taken) to find a search point in the highest level. Traditionally it is assumed that this highest level contains only optimal solutions, but this restriction can be omitted and the theorem then gives bounds for the time needed to reach a solution having at least a certain fitness.

We also note that, in principle, there is no need to restrict the fitness level method to fitness levels. Any partition of the set of possible states of the algorithm into a sequence of subsets such that the algorithm cannot leave the current state to a state in a lower-order set would be sufficient. We analyze an application of this type in Section~\ref{sec:applLO}, but stick to the usual fitness level language in the remainder. 

\subsection{Domination-Version of the Fitness Level Theorem}

We shall now show that under the same assumptions, a much stronger statement is valid, namely that the runtime $T$ is dominated by $\sum_{i=1}^{m-1} \Geom(p_i)$, that is, a sum independent random variables following geometric distributions with success probabilities $p_i$ (see Section~\ref{sec:geomdef} for a definition of the geometric distribution). This result appears to be very natural and was used without proof in~\cite{ZhouLLH12}, yet its proof requires the non-trivial Lemma~\ref{lprobdomind}. 

\begin{theorem}[Domination version of the fitness level method]\label{tproblevel}
  Consider an elitist evolutionary algorithm $\calA$ maximizing a function $f \colon \Omega \to \R$. Let $A_1, \dots, A_m$ be a partition of $\Omega$ such that for all $i, j \in [1..m]$ with $i < j$ and all $x \in A_i$, $y \in A_j$, we have $f(x) < f(y)$. Set $A_{\ge i} := A_{i} \cup \dots \cup A_{m}$. Let $p_1, \dots, p_{m-1}$ be such that for all $i \in [1..m-1]$ we have that if the best search point in the current parent population is contained in  $A_i$, then independently of the past with probability at least $p_i$ the next parent population contains a search point in $A_{\ge i+1}$. 
  
  Denote by $T$ the (random) number of iterations $\calA$ takes to generate a search point in $A_m$. Then \[T \preceq \sum_{i=1}^{m-1} \Geom(p_i),\] where this sum is to be understood as a sum of independent geometric distributions. 
\end{theorem}

\begin{proof}
  Consider a run of the algorithm $\calA$. For all $i \in [1..m]$, let $T_i$ be the first time (iteration) when $\calA$ has generated a search point in $A_{\ge i}$. Then $T = T_m = \sum_{i=1}^{m-1} (T_{i+1} - T_{i})$. By assumption, $T_{i+1} - T_i$ is dominated by a geometric random variable with parameter $p_i$ regardless what happened before time $T_i$. Consequently, Lemma~\ref{lprobdomind} gives the claim.
\end{proof}

\subsection{Chernoff Bounds for Sums of Independent Geometric Random Variables}\label{sec:chernoffgeom}

By Lemma~\ref{lprobdomexp}, the expected runtime in Theorem~\ref{tproblevel} satisfies $E[T] \le \sum_{i=1}^{m-1} \frac 1{p_i}$, which is the common version of the fitness level theorem~\cite{Wegener01}. However, by using tail bounds for sums of independent geometric random variables, we also obtain runtime bounds that hold with high probability. This was first proposed in~\cite{ZhouLLH12}, but did not give very convincing results due to the lack of good tail bounds at that time. We briefly present the tail bounds known by now and then give a few examples how to use them together with the new fitness level theorem.

\begin{theorem}\label{tprobcgeomungleich}
 Let $X_1, \ldots, X_n$ be independent geometric random variables with success probabilities $p_1, \dots, p_n$. Let $p_{\min} := \min\{p_i \mid i \in [1..n]\}$. Let $X:=\sum_{i=1}^n X_i$ and $\mu = E[X] = \sum_{i=1}^n \frac 1 {p_i}$. 
  \begin{enumerate}
  \item For all $\delta \ge 0$, 
  \begin{align}
  \Pr[X \ge (1+\delta)\mu] &\le \frac{1}{1+\delta} \, (1-p_{\min})^{\mu(\delta-\ln(1+\delta))}\label{eq:probcgeomUjanson1}\\
  &\le \exp(-p_{\min} \mu (\delta - \ln(1+\delta)))\label{eq:probcgeomUjanson2}\\ 
  &\le \left(1 + \frac{\delta \mu p_{\min}}{n}\right)^n \exp(-\delta\mu p_{\min})\label{eq:probcgeomUscheideler}\\
  &\le \exp\left(-\, \frac{(\delta \mu p_{\min})^2}{2n (1+\frac{\delta \mu p_{\min}}{n})} \right)\label{eq:probcgeomUweak}.
  \end{align}
  \item For all $0 \le \delta \le 1$,  
  \begin{align}
  \Pr[X \le (1-\delta) \mu] &\le (1-\delta)^{p_{\min} \mu} \exp(-\delta p_{\min} \mu)\label{eq:probcgeomLjanson}\\
  &\le \exp\left(- \frac{\delta^2 \mu \pmin}{2 - \frac 43 \delta}\right)\label{eq:probcgeomLmiddle}\\
  &\le \exp(- \tfrac 12 \delta^2 \mu \pmin)\label{eq:probcgeomLscheideler}.
  \end{align}
  \end{enumerate}
\end{theorem}
Estimates~\eqref{eq:probcgeomUjanson1} and~\eqref{eq:probcgeomUjanson2} are from~\cite{Janson17}, bound~\eqref{eq:probcgeomUscheideler} is from~\cite{Scheideler00}, and~\eqref{eq:probcgeomUweak} follows from the previous by standard estimates. The lower tail bound~\eqref{eq:probcgeomLjanson} is from~\cite{Janson17}. It implies~\eqref{eq:probcgeomLmiddle} via standard estimates. Inequality~\eqref{eq:probcgeomLscheideler} appeared in~\cite{Scheideler00}. 

It is surprising that none of these useful bounds appeared in a reviewed journal. For the case that all geometric random variables have the same success probability~$p$, the bound
\begin{align}
  \Pr[X \ge (1+\delta)\mu]  \le \exp\left(-\frac{\delta^2}{2}\frac{n-1}{1+\delta}\right)\label{eq:probcgeomgleichU}
\end{align}
appeared in~\cite{DoerrHK11}.

The bounds of Theorem~\ref{tprobcgeomungleich} allow the geometric random variables to have different success probabilities, however, the tail probability depends only on the smallest of them. This is partially justified by the fact that the corresponding geometric random variable has the largest variance, and thus might be most detrimental to the desired strong concentration. If the success probabilities vary significantly, however, then this approach gives overly pessimistic tail bounds. Witt~\cite{Witt14} proves the following result, which can lead to stronger estimates.

\begin{theorem}\label{tprobchernoffgeomwitt}
  Let $X_1, \dots, X_n$ be independent geometric random variables with success probabilities $p_1, \dots, p_n$. Let $X = \sum_{i=1}^n X_i$, $s = \sum_{i=1}^n (\frac 1 {p_i})^2$, and $p_{\min} := \min\{p_i \mid i \in [1..n]\}$. Then for all $\lambda \ge 0$, 
\begin{align}
  &\Pr[X \ge E[X] + \lambda] \le \exp\left(-\frac 14 \min\left\{\frac{\lambda^2}{s}, \lambda p_{\min}\right\}\right),\\
  &\Pr[X \le E[X] - \lambda] \le \exp\left(-\frac{\lambda^2}{2s}\right).
\end{align}  
\end{theorem}

As we shall see, we often encounter sums of independent geometrically distributed random variables $X_1, \dots, X_n$ with success probabilities $p_i$ proportional to $i$. For this case, the following result from~\cite{DoerrD15tight} gives stronger tail bounds than the previous result. Recall that the harmonic number $H_n$ is defined by $H_n = \sum_{i=1}^n \frac 1i$.
\begin{theorem}\label{tprobgeomharmonic}
  Let $X_1, \ldots, X_n$ be independent geometric random variables with success probabilities $p_1, \dots, p_n$. Assume that there is a number $C \le 1$ such that $p_i \ge C \frac in$ for all $i \in [1..n]$. Let $X = \sum_{i=1}^n X_i$. Then 
  \begin{align}
  &E[X] \le \tfrac 1C n H_n \le \tfrac 1C n (1 + \ln n),\\
  &\Pr[X \ge (1+\delta) \tfrac 1C n \ln n] \le n^{-\delta} \text{ for all } \delta \ge 0.
  \end{align}
\end{theorem}

A brief look at the last lines of the proof of this result in~\cite{DoerrD15tight} shows that also the following, minimally stronger tail bound holds for all $\lambda \in \N$.
\[\Pr[X \ge \tfrac 1C n \ln n + \lambda] \le n \left(1 - \frac Cn\right)^{\frac 1C n \ln n + \lambda -1} < \left(1 - \frac Cn\right)^{\lambda-1}\]
This immediately gives the domination result
\begin{equation}
X \preceq \tfrac 1C n \ln n + \Geom(\tfrac Cn).\label{probeqccdom}
\end{equation}

This domination result allows to show tail bounds for sums of independent random variables having a distribution as in Theorem~\ref{tprobgeomharmonic}. While we shall not need this in this work, the regular occurrence of such distributions in runtime analysis justifies showing the following result.
\begin{theorem}\label{tprobCneu}
  Let $Y_1, \dots, Y_m$ be independent, not necessarily identically distributed, and let each one satisfy the assumptions made on $X$ in Theorem~\ref{tprobgeomharmonic}. Let $Y = \sum_{i=1}^m Y_i$. Then $E[Y] \le \frac 1C n m H_n \le \frac 1C (\ln(n)+1) n m$ and 
  \[\Pr[Y \ge \tfrac 1C (\ln(n)+1) n m + \lambda] \le \exp\left(-\frac{\lambda^2 C^2}{2 n^2 m (1 + \frac{\lambda C}{nm})}\right).\]
\end{theorem}

\begin{proof}
  Let $p = \frac Cn$ and let $Z$ be the sum of $m$ independent $\Geom(p)$ random variables. Since $Y_i \preceq \tfrac 1C n \ln n + \Geom(p)$ by~\eqref{probeqccdom}, Lemma~\ref{lprobdomsum} shows $Y \preceq \tfrac 1C mn \ln n + Z$. Hence for all $\lambda \ge 0$ we compute
\begin{align*}
\Pr[Y \ge \tfrac 1C (\ln(n)+1) n m + \lambda] &\le \Pr[Z \ge E[Z] + \lambda]\\
&\le \exp\left(- \frac{\lambda^2 p^2}{2 m (1 + \frac{\lambda p}{m})} \right),
\end{align*}
where the last estimate follows from~\eqref{eq:probcgeomUweak}.
\end{proof}

We formulate explicitly the following corollary for sums of independent identical coupon collector distributions. In the basic coupon collector process, there are $n$ types of coupons. In each round, we obtain a coupon of a random type. The question is how many rounds are needed to have at least one coupon of each type. As a minimal extension of the very basic version, let us assume that we start this process having already coupons of $n-k$ different types, that is, we are missing only $k$ types. Let then $D_n^k$ be the distribution describing the number of rounds taken until we have a coupon of each type. Clearly, $D_n^k = \sum_{i=1}^k \Geom(\frac in)$. By the above theorem (taking $k$ as $n$ and $C$ as $\frac kn$), we obtain the following result.

\begin{corollary}
  Let $m, n \in \N$, $k \in [1..n]$, and $Y$ be the sum of $m$ independent $D_n^k$ coupon collector distributions. Then 
  \begin{align*}
  &E[Y] = m n H_k \le m n (\ln(k)+1),\\
  &\Pr[Y \ge m n (\ln(k)+1) + \delta n] \le \exp\left(- \frac{\delta^2}{2m(1+\frac{\delta}{m})}\right).
  \end{align*}
\end{corollary}

\subsection{Applications of the Fitness Level Theorem}\label{secfitnessappl}

We now present several examples where our fitness level theorem gives an upper bound in the domination sense and where this, often together with the just presented tail bounds, leads to new tail bounds for the runtime of different evolutionary algorithms. 

\subsubsection{General Upper Bound for Mutation-based Algorithms}

Consider the optimization of an arbitrary $f\colon \{0,1\}^n \to \R$ via any evolutionary algorithm $\calA$ which creates its offspring as random search points and via standard-bit mutation with mutation rate $p \le \frac 12$ from previous search points. Then for any such offspring the probability that it is an optimal solution is at least $p^{n}$, regardless of the history of the optimization process. Consequently, the runtime $T$ of $\calA$ on $f$ satisfies \[T \preceq \Geom(p^{n}),\] which implies
\begin{align*}
  &E[T] \le p^{-n},\\
  &\Pr[T \ge  \gamma p^{-n}+1] \le (1-p^{n})^{\gamma p^{-n}} \le e^{-\gamma} \mbox{ for all } \gamma \ge 0.
\end{align*}
Note that here the tail bound follows immediately from the definition of the geometric distribution. The result on the expectation was shown for the \oea with mutation rate $p=\frac 1n$ already in the seminal paper~\cite{DrosteJW02} with essentially the same argument.

\subsubsection{\texorpdfstring{Performance of the \oea and \oplea on $\onemax$}{Performance of the (1+1) EA and (1+lambda) EA on ONEMAX}} 

How the \oea optimizes the \onemax test function (see Section~\ref{sec:onemaxdef} for its definition) is one of the first results in runtime analysis. Using the fitness level method with the levels $A_i := \{x \in \{0,1\}^n \mid \OM(x) = i\}$, $i = 0,1,\dots, n$, one easily obtains that the expectation of the runtime $T$ is at most 
\begin{equation}
  E[T] \le en H_n \le en(\ln(n)+1).\label{eq:onemax}
\end{equation}
For this, it suffices to compute that the probability to leave the $i$-th fitness level satisfies $p_i \ge \frac {n-i}n (1-\frac1n)^{n-1} \ge \frac{n-i}{en}$, see~\cite{DrosteJW02} for the details. By Theorems~\ref{tproblevel} and~\ref{tprobgeomharmonic}, we also obtain
\begin{align*}
&T \preceq \sum_{i=1}^n \Geom\bigg(\frac{i}{en}\bigg),\\
&\Pr[T \ge (1+\delta) en \ln n] \le n^{-\delta} \mbox{ for all } \delta \ge 0.
\end{align*}
The domination result, while not very deep, appears to be new, whereas the tail bound was proven before via multiplicative drift analysis~\cite{DoerrG13algo}. 

Things become more interesting (and truly new) when regarding the \oplea instead of the \oea. We denote by $d(x) = n - \OM(x)$ the distance of $x$ to the optimum. Let $t = \lfloor \frac{\ln(\lambda)-1}{2 \ln\ln \lambda} \rfloor$. We partition the lowest $L = \lfloor n-\frac n {\ln \lambda} \rfloor$ fitness levels (that is, the search points with $\OM(x) < L$) into $\lceil \frac{L}{t} \rceil$ sets each spanning at most $t$ consecutive fitness levels. In~\cite{DoerrK15}, it was shown that the probability to leave such a set in one iteration is at least $p_0 \ge (1-\frac 1e)$. The remaining search points are partitioned into sets of points having equal fitness, that is, $A_i = \{x \mid \OM(x) = i\}$, $i = L, \dots, n$. For these levels, the probability to leave a set in one iteration is at least 
\[p_i \ge 1 - \left(1 - \left(1 - \frac 1n\right)^{n-1}\frac{n-i}{n}\right)^\lambda \ge 1 - \left(1 - \frac{n-i}{en}\right)^\lambda.\] 
For $i \le n - \frac{en}{\lambda}$, this is at least $p_i \ge 1 - \frac 1e$ by the well-known estimate $1+r \le e^r$ valid for all $r \in \R$. For $i > n - \frac{en}{\lambda}$, we use a new version of the Weierstrass product inequality (Lemma~4.8 in~\cite{Doerr18bookchapter}, stating that $\prod_{i=1}^n (1-x_i) \le 1 - S + \frac 12 S^2$ for all $x_1, \dots, x_n \in [0,1]$ and $S = \sum_{i=1}^n x_i$) to estimate $p_i \ge 1 - (1 - \frac{\lambda}{en}(n-i) + \frac 12 (\frac{\lambda}{en}(n-i))^2) \ge \frac 12 \frac{\lambda}{en}(n-i)$. Using the abbreviation $T_0 := \lceil \frac{L}{t} \rceil + \lceil \frac{n}{\ln \lambda}\rceil - \lceil \frac {en}{\lambda} \rceil + 1$, our fitness level theorem gives the following domination bound for the number $T$ of \emph{iterations} until the optimum is found.
\begin{align}
&T \preceq \sum_{i=1}^{T_0} \Geom(1-\tfrac 1e) + \sum_{i=1}^{\lceil\frac {en}\lambda - 1\rceil} \Geom(\tfrac 12 \tfrac{\lambda i}{en}).\label{eq:lea}
\end{align}
To avoid uninteresting case distinctions, let us assume that $\lambda = \omega(1)$ and $\lambda = n^{O(1)}$. In this case, the above domination statement immediately gives 
\[E[T] \le \frac e {e-1} T_0 + 2e \frac{n\ln(\lceil\frac {en}\lambda \rceil)}{\lambda} = (1+o(1)) \left(\frac{2e}{e-1} \frac{n \ln\ln \lambda}{\ln \lambda} + 2e \frac{n\ln(n)}{\lambda}\right),\]
which is the result of~\cite{DoerrK15} with explicit constants. With Theorems~\ref{tprobcgeomungleich} and~\ref{tprobgeomharmonic}, equation~\eqref{eq:lea} also implies tail bounds, however, for general $\lambda$ these become quite technical. For this reason, we omit any further details. 

\subsubsection{\texorpdfstring{Performance of the \mpoea on \leadingones}{Performance of the (mu+1) EA on LeadingOnes}}\label{sec:applLO}

The runtime of the \mpoea on various problems was analyzed in~\cite{Witt06}. Among them, we exemplarily regard the \leadingones problem (see Section~\ref{sec:deflo} for its definition). We shall not need the assumption of a random initialization in this section.

Following the general proof idea of~\cite{Witt06}, we let \[M := \min\{\lceil n / \ln(en) \rceil, \mu\}.\] For all $a \in [0..n]$ and $b \in [1..M]$, we say that the \mpoea optimizing \leadingones is in state $(a,b)$ if the maximal fitness of an individual in the population is $a$ and if there are exactly $b$ individuals having this maximal fitness. We say that the algorithm is in state $(a,M)$ also when the maximal fitness in the population is $a$ and there are more than $M$ individuals with this fitness. We sort these states in lexicographic order, that is, we say that $(a',b')$ is better than $(a,b)$ if $a' > a$ or if $a' = a$ and $b' \ge b$. From the definition of the \mpoea, it is clear that the algorithm cannot go from a better state to a worse one. Consequently, by viewing the states as fitness levels, we can apply the fitness level theorem (Theorem~\ref{tproblevel}). We observe that for all $a < n$ and $b < M$, the probability of leaving the state $(a,b)$ to a better state is at least 
\[p_{ab} := \frac{b}{\mu}\left(1-\frac 1n\right)^{a} \ge \frac{b}{e\mu},\]
 since to leave the state it suffices to select one of the $b$ best individuals and to not flip any of the $a$ bits contributing to the fitness (this creates a new best individual which replaces one of the inferior individuals which are still in the population). The probability to leave the state $(a,M)$ with $a < n$ is at least 
\[p_{aM} := \frac{M}{\mu} \frac 1n \left(1-\frac 1n\right)^{a} \ge \frac{M}{en\mu}.\]
 Consequently, the runtime $T$ of the \mpoea on the \leadingones function satisfies
\begin{align*}
T & \preceq \sum_{a=0}^{n-1} \sum_{b=1}^M \Geom(p_{ab}),\\
E[T] & \le \sum_{a=0}^{n-1} \sum_{b=1}^M \frac 1 {p_{ab}} = en\mu (\ln(M-1)+1 + \tfrac{n}{M})\\
& \le en\mu (\ln(en) + \max\{\tfrac{n}{\mu},\ln(en)\}) \le en(2\mu\ln(en)+n),\\
\Pr[T \ge E[T] + \lambda] & \le \exp\left(-\frac 14 \min\left\{\frac{\lambda^2}{e^2 n^3 \mu^2 / M^2 + \frac{\pi^2}{6} e^2 n \mu^2}, \frac{\lambda M}{en\mu}\right\}\right) \\
& = \exp\left(-\frac 14 \min\left\{\frac{(1-o(1))\lambda^2 M^2}{e^2 n^3 \mu^2}, \frac{\lambda M}{en\mu}\right\}\right). \\
\end{align*}
The tail bound follows from using Theorem~\ref{tprobchernoffgeomwitt} with 
\[s = n \left(\left(\frac{en\mu}{M}\right)^2 + \sum_{b=1}^{M-1} \left(\frac{e\mu}{b}\right)^2\right) \le \frac{e^2 n^3 \mu^2}{M^2} + \frac{\pi^2}{6} e^2 n \mu^2 = (1+o(1)) \frac{e^2 n^3 \mu^2}{M^2}\]
and $\pmin = \min\{\frac{M}{en\mu},\frac{1}{e\mu}\} = \frac{M}{en\mu}$, noting that $M = o(n)$ and $M \le n$. 

Apart from a small improvement in the constants, stemming from slightly more careful estimates, the result on the expected runtime is the same as the one in~\cite{Witt06}, which is $E[T] \le 3en \max\{\mu \ln(en),n\}$, see Theorem~1 of~\cite{Witt06} and recall that we count the number of iterations, that is, we ignore the $\mu$ fitness evaluations of the initial individuals. The tail bound, as discussed earlier, is stronger than the one in~\cite{ZhouLLH12} due to the stronger Chernoff bounds for sums of geometric random variables which are available now. 

\subsubsection{Performance of the \texorpdfstring{\oea}{(1+1) EA} on Jump Functions}

For $k \in [1..n]$, the $n$-dimensional jump function $\jump_k \colon \{0,1\}^n \to \R$ was defined in~\cite{DrosteJW02} by 
\[
\jump_k(x) = 
\begin{cases}
\onemax(x)+k & \mbox{if $\onemax(x) \in [0..n-k] \cup \{n\}$,}\\
n - \onemax(x) & \mbox{otherwise.}
\end{cases}
\]
Exemplarily for the runtime $T_k$ of the \oea on $\jump_k$, the fitness level theorem with $P_k = n^{-k}(1-\frac 1n)^{n-k}$ immediately yields 
\[T \preceq \sum_{i=1}^{k-1} \Geom(\tfrac{n-i}{en}) + \sum_{i=0}^{n-k-1} \Geom(\tfrac{n-i}{en}) + \Geom(P_k),\]
where for $k \ge 2$ the last term is by far the most important one (so we do not care about the minor improvements of this bound, which are easy to obtain).
Interestingly, we have a nearly matching lower bound \emph{in the domination sense}. To avoid unnecessary technicalities, we restrict ourselves to the case $k \in [2..\frac n4]$, but similar results could be shown for larger values of $k$ as well.

\begin{lemma}
  Let $n \in \N$ and $k \in [2..\frac n4]$. Let $T_k$ be the runtime of the \oea on the function $\jump_k$. Then with $P_k = n^{-k}(1-\frac 1n)^{n-k}$ we have
  \[X \cdot \Geom(P_k) \preceq T_k,\]
where $X$ is a binary random variable with $\Pr[X = 1] = 1 - \exp(-\frac n8)$.
\end{lemma}

\begin{proof}
  A simple application of the additive Chernoff bound (e.g., Theorem~10.7 in~\cite{Doerr18bookchapter}) shows that with probability at least $1 - \exp(-\frac n8)$, the random initial search point has a \onemax value of at most $\frac 34 n \le n-k$. In this situation, the current search point of the \oea will have a \onemax value of at most $n-k$ as long as the optimum is not found. For any search point with \onemax value $v \le n-k$, the probability to reach the optimum in one step is $n^{-(n-v)} (1-\frac 1n)^v \le P_k$. Consequently, conditional on the initial search point having a \onemax value of at most $n-k$, the runtime satisfies $\Geom(P_k) \preceq T$, giving the claim.
\end{proof}

We note without proof that the results of this subsection can easily be extended to other elitist mutation-based algorithms, since the main argument that with high probability once the optimum has to be generated from a search point in Hamming distance at least $k$ remains valid.

\subsubsection{Performance of the \texorpdfstring{\oea}{(1+1) EA} on the Sorting Problem}

One of the first combinatorial problems regarded in the theory of evolutionary computation is how a combinatorial \oea sorts an array of length~$n$. One of several setups regarded in~\cite{ScharnowTW04} is modelling the sorting problem as the minimization of the number of inversions in the array. We assume that the \oea mutates an array by first determining a number $k$ according to a Poisson distribution with parameter $\lambda = 1$ and then performing $k+1$ random exchanges (to ease the presentation, we do not use the jump operations also employed in~\cite{ScharnowTW04}, but it is easy to see that this does not significantly change things). It is easy to see that exchanging two elements that are in inverse order reduces the number of inversions by at least one. Hence if there are $i$ inversions, then with probability $\frac 1e i \binom{n}{2}^{-1}$, the \oea inverts exactly one of the inversions and thus improves the fitness (note that $\Pr[k=0] = \frac 1e$). By our fitness level theorem, the runtime $T$ is dominated by the independent sum $\sum_{i=1}^{\binom{n}{2}} \Geom(\frac{i}{e\binom{n}{2}})$. Hence 
\begin{align*}
&E[T] \le e \binom{n}{2} H_{\binom{n}{2}} \le \frac{e}{2} n^2 (1+2\ln n),\\
&\Pr[T \ge (1+\delta) e n^2 \ln n] \le \binom{n}{2}^{-\delta},
\end{align*}
where again the statement on the expectation has appeared before in~\cite{ScharnowTW04} (with slightly different constants stemming from the use of a slightly different algorithm) and the tail bound (stemming from Theorem~\ref{tprobcgeomungleich}) is new.

We note without proof that a similar analysis would also prove a tail bound for the slighly faster $O(n^2)$ time evolutionary sorting algorithm of~\cite{DoerrH08}.

\subsubsection{Further Results}

In this section, we state a few more results to demonstrate the range of applicability of domination arguments. Since for each of them a detailed discussion with precise proofs would require a longer explanation of the problem and the precise evolutionary algorithm used, we just state the result and point the reader interested in the details to the literature. In all cases, the domination result follows directly from the proof of the original result on the expected runtime.

\paragraph{Eulerian cycles.} 
The Eulerian cycle problem asks for a cycle in an undirected graph that traverses each edge exactly once. Ending a series~\cite{Neumann08,DoerrHN07,DoerrKS07} of works discussing suitable representations for the Eulerian cycle problem, matchings in the adjacency lists were found to give very good results~\cite{DoerrJ07}. Among several algorithms proposed in~\cite{DoerrJ07}, we regard the \oea that uses perfect matchings in the adjacency lists as genotype and an edge-based mutation operator. Using a fitness level argument that regards the number of disjoint cycles in the individual, one observes that the runtime $T$ of this algorithm is dominated by the independent sum \[T \preceq \sum_{i=1}^{m/3} \Geom(\tfrac{i}{2em}).\] By Theorem~\ref{tprobgeomharmonic}, this implies 
\begin{align*}
  &E[T] \le 2emH_{m/3}, \\
  &\Pr[T \ge 2(1+\delta)em\ln \tfrac m3] \le (\tfrac m3)^{-\delta} \mbox{ for all $\delta \ge 0$.}
\end{align*}

\paragraph{Vertex covers.} In an undirected graph, a vertex cover is a set of vertices such that each edge contains at least one of them. In~\cite{OlivetoHY09}, it was shown that the classic \oea using a binary representation can efficiently find minimal vertex covers on paths. Since this was shown via a fitness level argument, we easily obtain that the runtime $T$ on a path of length $n$ is dominated by \[T \preceq \sum_{i=1}^{n-1} \Geom(\Omega((\tfrac in)^4)) =: \overline T,\] which gives an expectation of \[E[T] \le E[\overline T] = O(n^4)\] and, via Theorem~\ref{tprobchernoffgeomwitt}, a tail estimate of \[\Pr[T \ge E[\overline T] + \lambda] \le 
\exp\left(-\Omega\left(\tfrac{\lambda}{n^4}\right)\right).\]

\paragraph{The $(1 + (\lambda,\lambda))$ genetic algorithm.} The $(1 + (\lambda,\lambda))$ GA was proposed in~\cite{DoerrDE15} as an algorithm that, despite elitism, also profits from generating search points which are inferior to the current-best solution. All analyses of this algorithm suggest to use it with a mutation rate of $p = \frac{\lambda}{n}$ and a crossover bias $c = \frac{1}{\lambda}$, where $\lambda$ denotes the size of the two offspring populations used by the algorithm. We shall assume these parameters as well. 

Assume that $\lambda = \omega(1)$ as this is the situation in which the $(1 + (\lambda,\lambda))$ GA can outperform the \oea. Let $L = \lceil n \ln\ln(\lambda)/\ln(\lambda) \rceil$. Let $C_1, C_2$ be sufficiently large constants. Imitating the proof of the asymptotically tight runtime analysis in~\cite{DoerrD18}, we see that the runtime of the $(1 + (\lambda,\lambda))$ GA on the \onemax function, measured via the number of iterations, is dominated by $\sum_{i=1}^{C_1 L} \Geom(p_i)$ with $p_i = C_2(1 - (1-\frac in)^{\lambda^2/2})$ for $i \in [1..L]$ and $p_i = 1 - 1/e - o(1)$ for $i \in [L+1 .. C_2 L]$.

\section{Beyond the Fitness Level Theorem}\label{secbeyondfitness}

Above we showed that in all situations where the classic fitness level method can be applied we immediately obtain that the runtime is dominated by a suitable sum of independent geometric distributions. We now show that the domination-by-distribution argument is not restricted to such situations. As an example, we use another combinatorial problem from~\cite{ScharnowTW04}, the single-source shortest path problem, and obtain a simplified and more natural proof of the currently strongest runtime bound for this problem from~\cite{DoerrHK11}. We note without proof that similar arguments could be used in the analysis of other problems where the evolutionary algorithm builds up the optimal solution incrementally from structurally smaller solutions, such as other path problems~\cite{Theile09,DoerrJ10,DoerrHK12} or dynamic programming~\cite{DoerrENTT11}.

In~\cite{ScharnowTW04}, the single-source shortest path problem in a connected undirected graph $G = (V,E)$ with edge weights $w \colon E \to \N$ and source vertex $s \in V$ was solved via a \oea as follows. Individuals are arrays of pointers such that each vertex different from the source has a pointer to another vertex. If, for a vertex $v$, following the pointers gives a path from $v$ to $s$, then the length (=sum of weights of its edges) of this path is the fitness of this vertex; otherwise the fitness of this vertex is infinite. The fitness of an individual is the vector of the fitnesses of all vertices different from the source. In the selection step, an offspring is accepted if and only if all vertices have an at least as good fitness as in the parent. This is called a multi-objective formulation of the single-source shortest path problem in~\cite{ScharnowTW04}. Mutating an individual means choosing a number $k$ from a Poisson distribution with parameter $\lambda = 1$ and then changing $k+1$ pointers to random new targets. 

The main analysis argument in~\cite{ScharnowTW04} is that, due to the use of the multi-objective fitness, a vertex that is connected to the source via a shortest path (let us call such a vertex \emph{optimized} in the following) remains optimized for the remaining run of the algorithm. Hence we can perform a structural fitness level argument over the number of optimized vertices. The probability to increase this number is at least $p := \frac{1}{e(n-1)(n-2)}$ because there is at least one non-optimized vertex $v$ for which the next vertex $u$ on a shortest path from $v$ to $s$ is already optimized. Hence with probability $\frac 1e$ we have $k=0$, with probability $\frac1{n-1}$ we choose $v$, and with probability $\frac1{n-2}$ we rewire its pointer to $u$. This gives an expected runtime of at most \[E[T] \le (n-1)/p = e(n-1)^2(n-2).\] 

To obtain better bounds for certain graph classes, the number $n_i$ of vertices for which a shortest path with fewest edges consists of $i$ edges is defined in~\cite{ScharnowTW04}. With a fitness level argument similar to the one above,  it takes an expected time of at most $en^2 H_{n_1}$ to have all $n_1$-type vertices optimally connected to the source. 
After this, an expected number of at most $en^2 H_{n_2}$ iteration suffices to connect all $n_2$-vertices to the source. Iterating this argument, a runtime estimate of 
\[E[T] \le en^2 \sum_{i=1}^{n-1} H_{n_i} \le e n^2 \sum_{i=1}^{n-1} (\ln(n_i)+1)\] 
is obtained in~\cite{ScharnowTW04}. This expression remains of order $\Theta(n^3)$ in the worst case, but becomes, e.g., $O(n^2 \ell \log(\frac{2n}{\ell}))$ when each vertex can be connected to the source via a shortest path having at most $\ell$ edges. 

By arguments different from fitness levels and with some technical effort, this result was improved to $O(n^2 \max\{\log n, \ell\})$ in~\cite{DoerrHK11}. We now show that this improvement could have been obtained via domination arguments in a very natural way. 

Consider some vertex $v$ different from the source. Fix some shortest path $P$ from $v$ to $s$ having at most $\ell$ edges. Let $V_v$ be the set of (at most $\ell$) vertices on $P$ different from $s$. As before, in each iteration we have a probability of at least $p = \frac{1}{e(n-1)(n-2)}$ that a non-optimized vertex of $V_v$ becomes optimized. Consequently, the time $T_v$ to connect $v$ to $s$ via a shortest path is dominated by a sum of $\ell$ independent $\Geom(p)$ random variables. We conclude that there are random variables $X_{ij}$, $i \in [1..\ell]$, $j\in [1..n-1]$, such that
\begin{enumerate}
\item $X_{ij} \sim \Geom(p)$ for all $i \in [1..\ell]$ and $j \in [1..n-1]$, 
\item for all $j \in [1..n-1]$ the variables $X_{1j},\dots,X_{\ell j}$ are independent, and 
\item the runtime $T$ is dominated by $\max\{Y_j \mid j \in [1..n-1]\}$, where $Y_j:=\sum_{i=1}^{\ell} X_{ij}$ for all $j \in [1..n-1]$. 
\end{enumerate}

It remains to deduce from this domination statement a runtime bound. Let $\delta = \max\{\frac{4 \ln(n-1)}{\ell-1},\sqrt{\frac{4 \ln(n-1)}{\ell-1}}\}$. For all $j \in [1..n-1]$, by~\eqref{eq:probcgeomgleichU}, we estimate
\begin{align*}
\Pr[Y_j \ge (1+\delta) E[Y_j]] &\le \exp\left(-\frac 12 \frac{\delta^2}{1+\delta} (\ell-1)\right) \\
&\le \exp\left(-\frac 14 \min\{\delta^2,\delta\} (\ell-1)\right) \\
&\le \exp\left(-\frac 14 \frac{4 \ln(n-1)}{\ell-1} (\ell-1)\right) \\
&= \frac 1 {n-1}.
\end{align*}
For all $\eps > 0$, again by~\eqref{eq:probcgeomgleichU}, we compute 
\begin{align*}
\Pr[Y_j \ge (1+\eps)(1+\delta) E[Y_j]] &\le \exp\left(-\frac 12 \frac{(\delta+\eps+\delta\eps)^2}{(1+\delta)(1+\eps)} (\ell-1)\right) \\
&\le \exp\left(-\frac 12 \frac{\delta^2(1+\eps)^2}{(1+\delta)(1+\eps)} (\ell-1)\right) \\
&\le \exp\left(-\frac 12 \frac{\delta^2}{1+\delta} (\ell-1)\right)^{1+\eps} \\
&\le (n-1)^{-(1+\eps)}.
\end{align*}

Recall that the runtime $T$ is dominated by $Y = \max\{Y_j \mid j \in [1..n-1]\}$, where we did not make any assumption on the correlation of the $Y_j$. In particular, they do not need to be independent. Let 
\[T_0 = (1+\delta)E[Y_1] = (1+\delta)\frac{\ell}{p}.\] 
Then \[\Pr[Y \ge (1+\eps)T_0] \le \sum_{j=1}^{n-1} \Pr[Y_j \ge (1+\eps)T_0] \le (n-1)^{-\eps}\] by the union bound. Transforming this tail bound into an expectation via standard arguments, e.g., Corollary~6.2~(c) in~\cite{Doerr18bookchapter}, we obtain $E[Y] \le (1 + \frac 1 {\ln(n-1)}) T_0$. We thus have shown the following result.

\begin{theorem}
  Let $G$ be an undirected graph on $n$ vertices together with integral edge weights. Let $s$ be a vertex. Assume that all vertices are connected to $s$ via a shortest path consisting of at most $\ell$ edges. Let $T$ be the runtime of the EA proposed in~\cite{ScharnowTW04} to solve the single-source shortest path problem in $G$. Then, with $\delta = \max\{\frac{4 \ln(n-1)}{\ell-1},\sqrt{\frac{4 \ln(n-1)}{\ell-1}}\}$, $p = \frac{1}{e(n-1)(n-2)}$, and $T_0 := (1+\delta) \frac \ell p$, we have
\begin{align*}
&E[T] \le \left(1+\frac{1}{\ln(n-1)}\right) T_0,\\
&\Pr[T \ge (1+\eps) T_0] \le (n-1)^{-\eps}
\end{align*}
for all $\eps \ge 0$.
\end{theorem}

This result is of same asymptotic order as the bound in~\cite{DoerrHK11}. We see the main advantage of the proof above in the fact that it is less technical and more natural than the previous one. However, our analysis also gives a better leading constant. For $\ell \gg \log(n)$, for example, our bound on the expected runtime is $(1+o(1)) e \ell n^2$, whereas it is $8 (1+o(1)) e \ell n^2$ in~\cite{DoerrHK11}.

\section{Structural Domination}\label{seclower}

So far we have used stochastic domination to compare runtime distributions. We now show that stochastic domination can be a very useful tool also to express structural properties of the optimization process. As an example, we give a short and elegant proof for the result of Witt~\cite{Witt13} that compares the runtimes of mutation-based algorithms. The main reason why our proof is significantly shorter than the one of Witt is that we use the notion of stochastic domination also for the distance from the optimum. 

To state this result, we need the notion of a \emph{$(\mu,p)$ mutation-based algorithm} introduced in~\cite{Sudholt13}. This class of algorithms is  called only \emph{mutation-based} in~\cite{Sudholt13}, but since (i)~it does not include all adaptive algorithms using mutation only, e.g., those regarded in~\cite{JansenW06,OlivetoLN09,BottcherDN10,BadkobehLS14,DangL16self,DoerrGWY17,DoerrWY18}, (ii)~it does not include all algorithms using a different mutation operator than standard-bit mutation, e.g., those in~\cite{DoerrDY16PPSN,DoerrDY16gecco,LissovoiOW17,DoerrLMN17}, and (iii)~this notion collides with the notion of unary unbiased black-box complexity algorithms (see~\cite{LehreW12}), which with some justification could also be called the class of mutation-based algorithms, we feel that a notion making these restrictions precise is more appropriate.

The class of $(\mu,p)$ mutation-based algorithms comprises all algorithms which first generate a set of $\mu$ search points uniformly and independently at random from $\{0,1\}^n$ and then repeat generating new search points from any of the previous ones via standard-bit mutation with probability~$p$ (that is, by flipping bits independently with probability~$p$). This class includes all $(\mu+\lambda)$ and $(\mu,\lambda)$ EAs which only use standard-bit mutation with static mutation rate~$p$.

Denote by \oeamu the following algorithm in this class. It first generates $\mu$ random search points. From these, it selects uniformly at random one with highest fitness and then continues from this search point as a \oea, that is, repeatedly generates a new search point from the current one via standard-bit mutation with rate $p$ and replaces the previous one by the new one if the new one is not worse (in terms of the fitness). This algorithm was called \oea with BestOf($\mu$) initialization in~\cite{LaillevaultDD15}.

For any algorithm $\calA$ from the class of $(\mu,p)$ mutation-based algorithms and any fitness function $f \colon \{0,1\}^n \to \R$, let us denote by $T(\calA,f)$ the runtime of the algorithm $\calA$ on the fitness function $f$, that is, the number of the first individual that is an optimal solution. Usually, this will be $\mu$ plus the number of the iteration in which the optimum was generated. To cover also the case that one of the random initial individuals is optimal, let us assume that these initial individuals are generated sequentially. 

In this language, Witt~\cite{Witt13} shows the following remarkable result.
\begin{theorem}\label{tprobdom}
  For any $(\mu,p)$ mutation-based algorithm $\calA$ and any $f \colon \{0,1\}^n \to \R$ with unique global optimum, \[T(\mbox{\oeamu},\onemax) \preceq T(\calA,f).\]
\end{theorem}
This result significantly extends results of a similar flavor in~\cite{BorisovskyE08,DoerrJW12algo,Sudholt13}. The importance of such types of results is that they allow to prove lower bounds for the performance of many algorithm on essentially arbitrary fitness functions by just regarding the performance of the \oeamu on \onemax. 

Let us denote by $|x|_1$ the number of ones in the bit string $x \in \{0,1\}^n$. Using similar arguments as in~\cite[Section~5]{DrosteJW00} and~\cite[Lemma~13]{DoerrJW12algo}, Witt~\cite{Witt13} shows the following natural domination relation between offspring generated via standard-bit mutation.
\begin{lemma}\label{lprobdommut}
  Let $x, y \in \{0,1\}^n$. Let $p \in [0,\frac 12]$. Let $x', y'$ be obtained from $x, y$ via standard-bit mutation with rate $p$. If $|x|_1 \le |y|_1$, then $|x'|_1 \preceq |y'|_1$.
\end{lemma}

We are now ready to give our alternate proof for Theorem~\ref{tprobdom}. While it is clearly shorter that the original one in~\cite{Witt13}, we also feel that it is more natural. In very simple words, it shows that $T(\calA,f)$ dominates $T(\mbox{\oeamu},\onemax)$ because the search points generated in the run of the \oeamu on $\onemax$ always are at least as close to the optimum (in the domination sense) as in the run of $\calA$ on $f$, and this follows from the previous lemma, a suitable coupling, and induction.

\begin{proof}[Proof of Theorem~\ref{tprobdom}]
  As a first small technicality, let us assume that the \oeamu in iteration $\mu+1$ does not choose a random optimal search point, but the last optimal search point. Since all the first $\mu$ individuals are generated independently, this modification does not change anything.

  Since $\calA$ treats bit-positions and bit-values in a symmetric fashion, we may without loss of generality assume that the unique optimum of $f$ is $(1,\dots,1)$. 
  
  Let $x^{(1)}, x^{(2)}, \dots$ be the sequence of search points generated in a run of $\calA$ on the fitness function $f$. Hence $x^{(1)}, \dots, x^{(\mu)}$ are independently and uniformly distributed in $\{0,1\}^n$ and all subsequent search points are generated from suitably chosen previous ones via standard-bit mutation with rate $p$. Let $y^{(1)}, y^{(2)}, \dots$ be the sequence of search points generated in a run of the \oeamu on the fitness function $\onemax$. 
  
  We show how to couple these random sequences of search points in a way that $|\tilde x^{(t)}|_1 \le |\tilde y^{(t)}|_1$ for all $t \in \N$. We take as common probability space $\Omega$ simply the space that $(x^{(t)})_{t \in \N}$ is defined on and let $\tilde x^{(t)} = x^{(t)}$ for all $t \in \N$. 
  
  We define the $\tilde y^{(t)}$ inductively as follows. For $t \in [1..\mu]$, let $\tilde y^{(t)} = x^{(t)}$. Note that this trivially implies $|\tilde x^{(t)}|_1 \le |\tilde y^{(t)}|_1$ for these search points. Let $t > \mu$ and assume that $|\tilde x^{(t')}|_1 \le |\tilde y^{(t')}|_1$ for all $t' < t$. Let $s \in [1..t-1]$ be maximal such that $|\tilde y^{(s)}|_1$ is maximal among $|\tilde y^{(1)}|_1, \dots, |\tilde y^{(t-1)}|_1$. Let $r \in [1..t-1]$ be such that $x^{(t)}$ was generated from $x^{(r)}$ in the run of $\calA$ on $f$. By induction, we have $|x^{(r)}|_1 \le |\tilde y^{(r)}|_1$. By choice of $s$ we have $|\tilde y^{(r)}|_1 \le |\tilde y^{(s)}|_1$. Consequently, we have $|x^{(r)}|_1 \le |\tilde y^{(s)}|_1$. By Lemma~\ref{lprobdommut} and Theorem~\ref{tprobdomcou}, there is a random $\tilde y^{(t)}$ (defined on $\Omega$) such that $\tilde y^{(t)}$ has the distribution of being obtained from $\tilde y^{(s)}$ via standard-bit mutation with rate $p$ and such that $|x^{(t)}|_1 \le |\tilde y^{(t)}|_1$. 
  
  With this construction, the sequence $(\tilde y^{(t)})_{t \in \N}$ has the same distribution as $(y^{(t)})_{t \in \N}$. This is because the first $\mu$ elements are random and then each subsequent one is generated via standard-bit mutation from the current-best one, which is just the way the \oeamu is defined. At the same time, we have $|\tilde x^{(t)}|_1 \le |\tilde y^{(t)}|_1$ for all $t \in \N$. Consequently, we have $\min\{t \in \N \mid |\tilde y^{(t)}|_1 = n\} \le \min\{t \in \N \mid |x^{(t)}|_1 = n\}$. Since  $T(\mbox{\oeamu},\onemax)$ and $\min\{t \in \N \mid |\tilde y^{(t)}|_1 = n\}$ are identically distributed and also $T(\calA,f)$ and $\min\{t \in \N \mid |x^{(t)}|_1 = n\}$ are identically distributed, we have $T(\mbox{\oeamu},\onemax) \preceq T(\calA,f)$. 
\end{proof}

\section{Counter-Examples}

Stochastic domination is a strong property. Therefore, it is easy to encounter situations in which the intuitive feeling that one process is faster than a second one does not imply stochastic domination. Here are a few examples.

\paragraph{Random search vs. \oea on \onemax.} Let us compare the performance of the random search heuristic (cf.~Section~\ref{sec:randomsearch}) and the \oea on the \onemax benchmark function. Both intuitive considerations and the known results on the expected runtimes ($2^n$ for random search, see Lemma~\ref{lem:randomsearch}, and at most $en(\ln(n)+1)$ for the \oea, see~\eqref{eq:onemax}) suggest that the \oea is much more efficient than random search. Nevertheless, the runtime $T_{RS}$ of random search does not dominate the runtime $T_{EA}$ of the \oea, as we show now. 

Note that to have a fair comparison between these two algorithms, we here regard as runtime truly the number of individuals evaluated until the optimum is first evaluated, that is, we do count the evaluation of the initial individual for the \oea. For either of the two algorithms, let $X_1$ and $X_2$ be the first two individuals generated (where we assume that the algorithm does not stop when $X_1$ is already the optimum). Denoting the optimal solution by $x^*$, the runtime $T$ of this algorithm satisfies 
\[\Pr[T \le 2] = \Pr[X_1 = x^*] + \Pr[X_2 = x^*] - \Pr[X_1 = X_2 = x^*].\] 
Note that $X_2$ also for the \oea is uniformly distributed in $\{0,1\}^n$. Hence 
\begin{align*}
  \Pr[T_{RS} \le 2] & = 2^{-n+1} - 2^{-2n} = (1+o(1)) 2^{-n+1},\\
  \Pr[T_{EA} \le 2] & = 2^{-n+1} - (1-\tfrac 1n)^n 2^{-n} = (1+o(1))(1 - \tfrac 1 {2e}) 2^{-n+1},
\end{align*}
contradicting the assertion that $T_{EA} \preceq T_{RS}$.

\paragraph{Fitness-proportionate selection.}
When optimizing objective functions with a strong fitness-distance correlation via a reasonable evolutionary algorithm, then it seems plausible that replacing individuals by better ones can only improve the performance. This is not true, for example, when fitness-proportionate selection is used, as the following example shows. 

Consider the optimization of the \onemax function, which has the perfect fitness-distance correlation. Assume we use an algorithm which at some time selects an individual $x$ from a population $P = \{x_1, \dots, x_\mu\}$, $\mu \ge 2$, via fitness-proportionate selection, that is, we have 
\[\Pr[x = x_i] = \frac{f(x_i)}{\sum_{j=1}^\mu f(x_i)},\] 
where $f = \onemax$. Now let $x_1', \dots, x_\mu'$ be search points with $f(x'_i) \ge f(x_i)$ for all $i \in [1..\mu]$. Let $P' = \{x_1', \dots, x_\mu'\}$ and assume that we select $x'$ from $P'$ via fitness-proportionate selection. Then we do not necessarily have $f(x) \preceq f(x')$. 

For an extreme counter-example, let $n$ be a multiple of $10$ and let the search points $x_i$ be such that $f(x_1) = 0.8n$ and $f(x_i) = 0$ for $i \in [2..\mu]$. Then 
\[\Pr[f(x) \ge 0.8n] = \Pr[x = x_1] = 1.\] 
Assume that $f(x'_i) = f(x_i)+0.1n$ for all $i \in [1..\mu]$. Then 
\[\Pr[f(x') \ge 0.8n] = \Pr[x' = x'_1] = \frac{0.8n+0.1n}{0.8n+\mu \cdot 0.1n} = \frac{9}{\mu+8} < 1.\].

\paragraph{Dependencies.} So far, we mostly experienced that stochastic domination is very well-behaved. If $X \preceq Y$, then all reasonable upper bound statements on $Y$ hold as well for $X$. Also, stochastic domination allowed us to compare two random variables irrespective of possible dependencies (in fact, we do not even need that they are defined over a common probability space). We now want to outline a point where some caution is required, and this is that dependencies between other random variables cannot be ignored. For example, $X_1 \preceq Y_1$ and $X_2 \preceq Y_2$ do not necessarily imply that 
\[\max\{X_1, X_2\} \preceq \max\{Y_1, Y_2\}.\] 
Whether such a statement is true depends on the correlation between the $X_i$ and the correlation between the~$Y_i$. For this reason, we needed to construct the coupling (which determines the dependencies) of the two processes in the proof of Theorem~\ref{tprobdom}.

As a simple counter-example, we can again regard the processes of running random search and the \oea on the \onemax function. Let $X_0, X_1, \dots$ denote the search points generated by random search and let $Y_0, Y_1, \dots$ denote the search points generated by the \oea. Clearly, each $X_i$ is uniformly distributed in $\{0,1\}$, hence $f(X_i) \sim \Bin(n,\frac 12)$. For the \onemax process, we have $f(Y_0) \sim \Bin(n,\tfrac 12)$. For all $i \ge 0$, we also have $f(Y_0) \preceq f(Y_i)$, and hence $f(X_i) \preceq f(Y_i)$. Consequently, we have $f(X_i) \preceq f(Y_i)$ for all $i \ge 0$. However, as shown above, we do not have $\max\{f(X_0), f(X_1)\} \preceq \max\{f(Y_0), f(Y_1)\}$.

\section{Conclusion}

In this work, we argued that stochastic domination can be very useful in runtime analysis, both to formulate more informative results and to obtain simpler and more natural proofs. We also showed that in many situations, in particular, whenever the fitness level method is applicable, it is easily possible to describe the runtime via a domination statement. 
 
We note however that not all classic proofs easily reveal details on the distribution. For results obtained via random walk arguments, e.g., the optimization of the short path function SPC$_n$~\cite{JansenW01}, monotone polynomials~\cite{WegenerW05}, or vertex covers on paths-like graphs~\cite{OlivetoHY09}, as well as for results proven via additive drift~\cite{HeY01}, the proofs often give little information about the runtime distribution. Note however that in~\cite{GarnierKS99} results on the convergence of the (suitably scaled) runtime distribution could be obtained.

For results obtained via the average weight decrease method~\cite{NeumannW07} or multiplicative drift analysis~\cite{DoerrG13algo}, the proofs also do not give information on the runtime distribution. However, the probabilistic runtime bound of type $\Pr[T \ge T_0 + \lambda] \le (1-\delta)^\lambda$ obtained from these methods implies that the runtime is dominated by $T \preceq T_0 -1 + \Geom(1-\delta)$. 

Overall, both from regarding these results and the history of the field, we suggest to more frequently formulate results via domination statements. Even in those cases where the probabilistic tools at the moment are not ready to exploit such a statement, there is a good chance future developments overcome this shortage and then it pays off if the result is readily available in a distribution form and not just as an expectation.

}

\newcommand{\etalchar}[1]{$^{#1}$}

\end{document}